\documentclass{amsart}
\usepackage{amsfonts}
\usepackage{bbm}
\usepackage{mathrsfs}
\usepackage{amsmath}
\usepackage{bigdelim}
\usepackage{multirow}
\usepackage{amssymb}
\usepackage{indentfirst}
\usepackage{CJK}
\usepackage{fancyhdr}
\usepackage{amscd,amssymb,amsmath,graphicx,verbatim,xy}
\usepackage{url}
\usepackage[TS1,OT1,T1]{fontenc}

\usepackage{rotating}    
\usepackage{float}
\usepackage{multirow} 
\usepackage{array}    
\usepackage{xcolor}
\usepackage{tikz}
\usepackage{tikz-cd}
\usepackage{bbm}
\usepackage{caption}
\usetikzlibrary{arrows.meta, positioning, calc}
\usepackage{makecell}
\usepackage{geometry}
\usepackage{graphics}
\usepackage{url}

\usepackage{hyperref}

\newtheorem{theorem}{Theorem}[section]

\newtheorem{proposition}[theorem]{Proposition}

\theoremstyle{definition}
\newtheorem{definition}[theorem]{Definition}
\newtheorem{example}[theorem]{Example}

\theoremstyle{remark}
\newtheorem{remark}[theorem]{Remark}
\numberwithin{equation}{section}

\begin{document}

\title[G-LoG Bi-filtration for Medical Image Classification]{G-LoG Bi-filtration for Medical Image Classification}

\author{Qingsong Wang}\thanks{Qingsong Wang and Jiaxing He contributed equally to this study.}
\address{Qingsong Wang, School of Mathematics, Jilin University, 130012, Changchun, P. R. China}
\email{qswang21@mails.jlu.edu.cn}

\author{Jiaxing He}
\address{Jiaxing He, School of Mathematics, Jilin University, 130012, Changchun, P. R. China}
\email{547337872@qq.com}

\author{Bingzhe Hou}
\address{Bingzhe Hou, School of Mathematics, Jilin University, 130012, Changchun, P. R. China}
\email{houbz@jlu.edu.cn}

\author{Tieru Wu}
\address{Tieru Wu, School of Mathematics, Jilin University, 130012, Changchun, P. R. China\\
School of Artificial Intelligence, Jilin University, 130012, Changchun, P. R. China}
\email{wutr@jlu.edu.cn}

\author{Yang Cao}
\address{Yang Cao, School of Mathematics, Jilin University, 130012, Changchun, P. R. China}
\email{caoyang@jlu.edu.cn}

\author{Cailing Yao}
\address{Cailing Yao, School of Mathematics, Jilin University, 130012, Changchun, P. R. China}
\email{1290279144@qq.com}
	
\date{}
\thanks{}
\begin{abstract}
Building  practical filtrations on objects to detect topological and geometric features is an important task in the field of Topological Data Analysis (TDA). In this paper, leveraging the ability of the Laplacian of Gaussian operator to 
enhance the boundaries of medical images, we define the G-LoG (Gaussian-Laplacian of Gaussian) bi-filtration to generate the features more suitable for multi-parameter persistence module. 
By modeling volumetric images as bounded functions, then we prove the interleaving distance on the persistence modules obtained from our bi-filtrations  on the bounded functions is stable with respect to the maximum norm of the bounded functions.  Finally, we conduct  experiments on the MedMNIST dataset, comparing our bi-filtration against single-parameter filtration and the established deep learning baselines, including Google AutoML Vision, ResNet, AutoKeras and auto-sklearn. Experiments results demonstrate that our bi-filtration significantly outperforms single-parameter filtration.  Notably, a simple Multi-Layer Perceptron (MLP) trained on the topological features generated by our bi-filtration achieves performance comparable to complex deep learning models trained on the original dataset.
\end{abstract}
\maketitle

\section{Introduction}
Over the last few years, Topological Data Analysis (TDA), especially persistent homology, has demonstrated significant utility in both theoretical and applied domains (see surveys \cite{zomorodian2005topology, Leo-2020, CV-2021} and the references therein).
Persistent homology is generated by a filtration, where different filtrations capture distinct topological features. 
In the single-parameter setting, four well-known filtrations are commonly used: the Vietoris-Rips, \v{C}ech and alpha filtrations for point clouds, and the lower-star filtration for cubical complexes.

Since a single filtration can not often  adequately capture enough structures, the concept of multi-parameter persistent homology  becomes crucial. Unlike the single persistence,  Carlsson and  Zomorodian  proved there is no analogous complete discrete invariant for the multi-parameter module \cite{Carls-2009}. Over the years, several methods have been proposed to construct multi-parameter filtrations for point clouds. For instance, Lesnick and McCabe focus on nerve models of subdivision bi-filtrations\cite{lesnick2024nerve}, Alonso et al. proposed the Delaunay bi-filtrations \cite{alonso2024delaunay} and Corbet et al.  proposed the rhomboid bi-filtration \cite{Re-2023}. These approaches can be found in \cite{Blum-2022} for more details.  

However, until now, methods that directly constructing bi-filtrations from images have remained scarce. In \cite{Math-2020}, Carri\`{e}re and Blumberg built bi-filtrations from a dataset of paired images, where each pair comprises imaging data of human tissue samples from patients with breast cancer. In \cite{JBTY-2025}, He et al. constructed mix-GENEO bi-filtration on MNIST and demonstrated the superiority of multi-parameter filtrations over single-parameter ones. Given that GENEO operators require careful selection in specific ways \cite{Bergo-2019},  there is a need for more accessible alternatives. Since the Gaussian filtered image eliminates the noise and Laplacian operator can detect edges of images, the Laplacian-of-Gaussian (LoG) has become a cornerstone operator in edge detection and texture enhancement, we propose a simpler and more efficient bi-parameter filtrations, named G-LoG. Then we evaluate their performances on the MedMNIST (v2) dataset \cite{YSWLZKPN-2023}. Figures \ref{fig:tissue2d} and \ref{fig:organ3d} illustrate our G-LoG bi-filtration method on 2D and 3D medical images, respectively. 

\begin{figure}[H]
	\centering
	\includegraphics[width=0.95\textwidth]{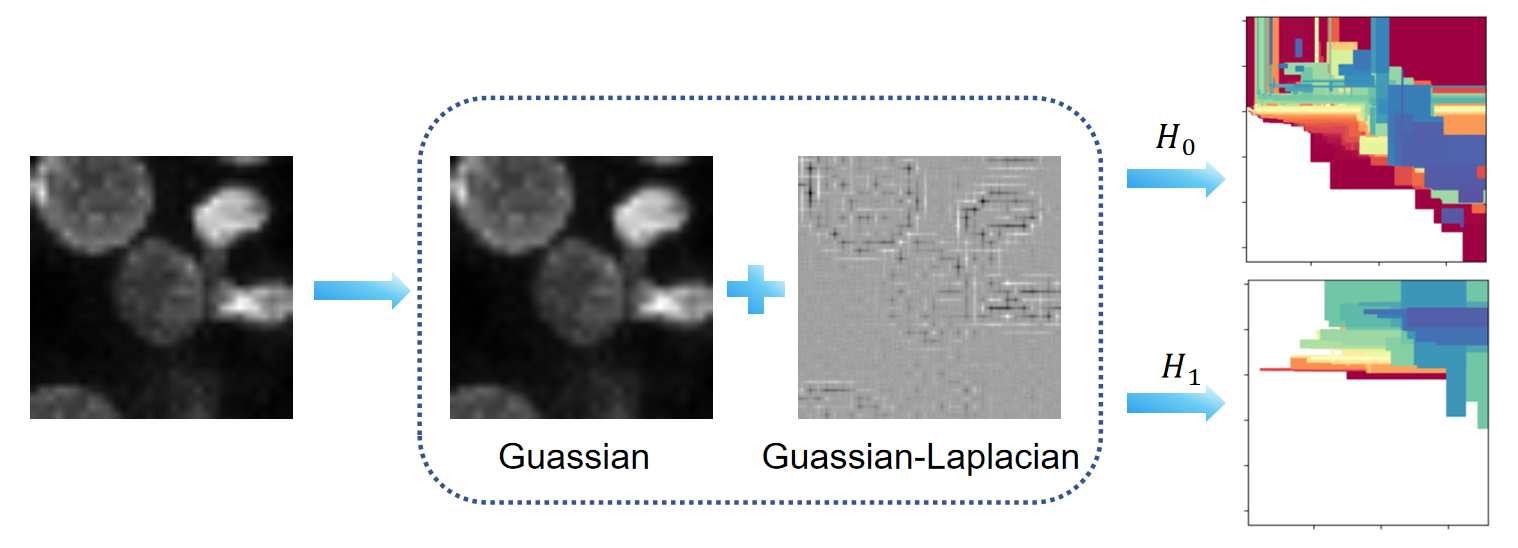}
	\caption{Bi-parameter persistence modules $H_0$ and $H_1$ generated by G-LoG bi-filtration from a medical tissue image.}
	\label{fig:tissue2d}
\end{figure}

\begin{figure}[H]
	\centering
	\includegraphics[width=0.95\textwidth]{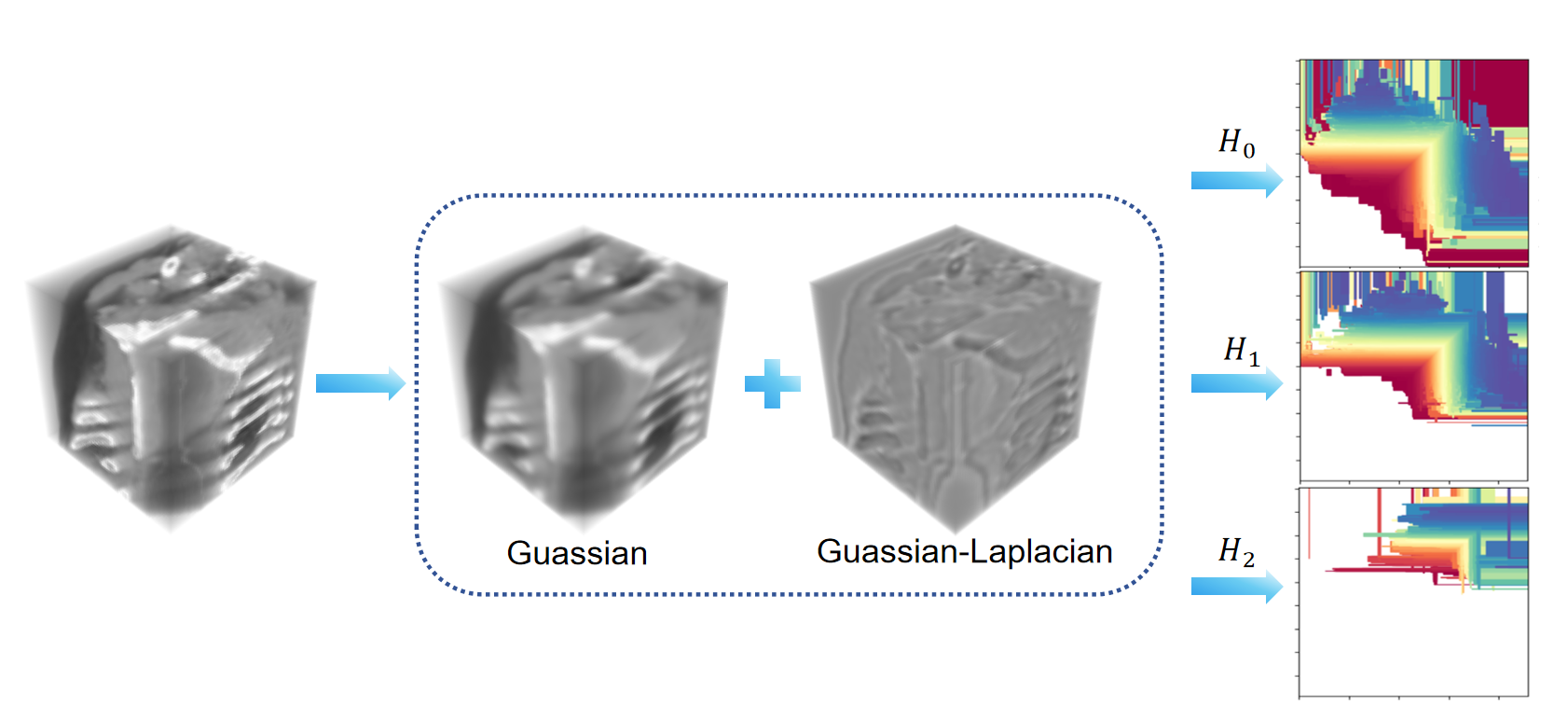}
	\caption{Bi-parameter persistence modules $H_0$, $H_1$ and $H_2$ generated by G-LoG bi-filtration from a 3D medical organ volume.}
	\label{fig:organ3d}
\end{figure}

\subsection{Overview}

In the recent years, deep learning has rapidly become a cornerstone of medical analysis. In medical image analysis, the use of deep learning to discover or learn informative features plays an essential role across a wide range of tasks. Nevertheless, the field still faces enduring challenges, including topological uncertainty \cite{SHDSELB-2025}, issues associated with high-dimensional data \cite{WHC-2021}, interpretability and the reliance on vast amounts of labeled data \cite{SK-2022}. Topological Data Analysis (TDA), a powerful method for characterizing the connectivity and intrinsic features of medical images, has attracted widespread attention.


Topological Data Analysis, which can be dated back to 1990s,  with methods such as Morse-Smale complexes, Reeb graphs and size theory being introduced and studied. During this period,  Ferri,  Frosini, Landi 
and their collaborators also proposed early versions of the rank invariant 
in degree $0$ and of persistent homotopy groups in multi-parameter persistence (see survey \cite{BDFFGLPS-2008} and the references therein). Soon after,  Edelsbrunner, Letscher, and  Zomorodian 
invented a version of homology theory named persistent homology which provides insights into the topological features that persist across
different scales \cite{Edel-Lets-Zo-2000}. Until now, TDA methods have been investigated in several medical fields, including neurology, cardiology, hepatology, genomics and single-cell transcriptomics, drug discovery, evolutionary and protein structure analysis \cite{SFHQLJCE-2023}. For example,  Hu et al. utilized discrete Morse theory to segment fine-grained structures in biomedical images \cite{HWFSC-2021}, while  Shen et al. introduced knot data analysis in \cite{SFLLWW-2024knot} for capturing local structures and connectivity in data.  Casaclang-Verzosa et al. \cite{CSKCTBABNM-2019} characterized the natural progression of aortic stenosis using TDA in cardiovascular research.

Persistent homology, as a ubiquitous tool in TDA which can be vectorized as the feature (such as persistence landscape \cite{Peter-2015}, persistence images \cite{Adams-Emer-Kir-2017}, betti curves \cite{CL-2022}, kernel methods \cite{RHBK-2015}, B-spline grids \cite{DLZZL-2024}) has been successfully used in medical image analysis. For example,  Crawford et al. designed the smooth Euler characteristic transform
to quantify magnetic resonance images  of tumors \cite{CMCMR-2020}.    Yadav et al. employed persistent homology  in histopathological cancer detection \cite{YADGC-2023}.

There are also several works considered  applications of multi-parameter persistence modules in medical image analysis. Vipond et al. employed the multi-parameter persistence landscape, constructed via radius-codensity bifiltrations, to study immune cell locations in digital histology images from head and neck cancer \cite{OJPMUCHH-2021}.  Carri\`{e}re and  Blumberg proposed multi-parameter persistence images \cite{Math-2020} and utilized this vectorization method to analyze quantitative immunofluorescence images.

The construction of most medical imaging datasets relies on prior knowledge from computer vision or the medical domain, making these datasets difficult to be used directly. To avoid this issue, two versions of the MedMNIST dataset \cite{YSWLZKPN-2023, medmnistv1} have been successively released. To establish a benchmark in this field, the developers employed several types of models as baseline classifiers for MedMNIST: ResNet \cite{HZRS-2016} and various AutoML models (including auto‑sklearn \cite{FEFLH-2022}, AutoKeras \cite{JSH-2019} and Google AutoML Vision \url{https://cloud.google.com/automl}). In recent years, many studies have aimed to improve classification performance on MedMNIST by modifying network architectures, such as FPVT \cite{LLCLC-2022}, MedViT \cite{MAKSA-2023} and C-Mixer \cite{SJZ-2025}.  In \cite{NKKC-2025},  Nuwagira et al. validated the effectiveness of Top-Med on MedMNIST, they showed that integrating topological feature vectors can enhance the accuracy and robustness of deep learning models.  Different from recent authors’ efforts to validating the learning ability of the network, we simply employ a multilayer perceptron (MLP) to classify features extracted from MedMNIST (v2) via multi-parameter persistent homology. In this paper, we use the results from MedMNIST (v2) \cite{YSWLZKPN-2023} and Top-Med \cite{NKKC-2025}  as baselines.

\subsection{Motivation}
Choosing an  appropriate multi-parameter filter function plays an important role in the classification task of multi-parameter persistent homology. If we select unsuitable filter function $\vec{\gamma}=(\gamma^1,\gamma^2)$ to obtain multi-parameter persistent homology, the multi-parameter persistent homology may yield results comparable to single-parameter persistent homology induced by $\gamma^1$ and $\gamma^2$ distinctly. Then we will fail to exploit the benefits of multi-parameter persistent homology. We construct a bi-parameter filter function for image processing based on the following three points.
\begin{enumerate}
	\item[(1)] To extract features more suitable for persistent homology, we employ two approaches: first, we use sublevel set functions to capture the original persistent homology features of the image; second, to better adapt the image for classification tasks within the persistent homology framework, we utilize the Laplacian operator to extract image edges. 
	\item[(2)] The level sets of the two parameters we select should intersect. If the multi-parameter filter functions we construct are ``independent", then multi-parameter filtration is essentially single-parameter in nature. We provide the following example to demonstrate that the bi-parameter persistence module along each direction can be decomposed into the direct sum of the single-parameter persistence modules under certain conditions (the relevant definitions for the example are provided in the Preliminaries).
	\begin{example}\textup{(Essential single-parameter)}
		
		Let $\mathcal{M}$ and $\mathcal{N}$ be two compact, disjoint $n$-dimensional submanifolds of $\mathbb{R}^{n}$. Let $f:\mathbb{R}^n\rightarrow\mathbb{R}$ and $g:\mathbb{R}^n\rightarrow\mathbb{R}$ be two continuous sublevel set functions satisfying that 
		\begin{align*}
			&f|_{\mathcal{M}}\geq 0, f|_{\mathbb{R}^{n}\setminus\mathcal{M}}=0,\\
			&g|_{\mathcal{N}}\geq 0,\  g|_{\mathbb{R}^{n}\setminus\mathcal{N}}=0.
		\end{align*} We provide Figure \ref{essential_1_para} as an example of the above definition.
		\begin{figure}[H]
			\centering
			\includegraphics[width=0.95\textwidth]{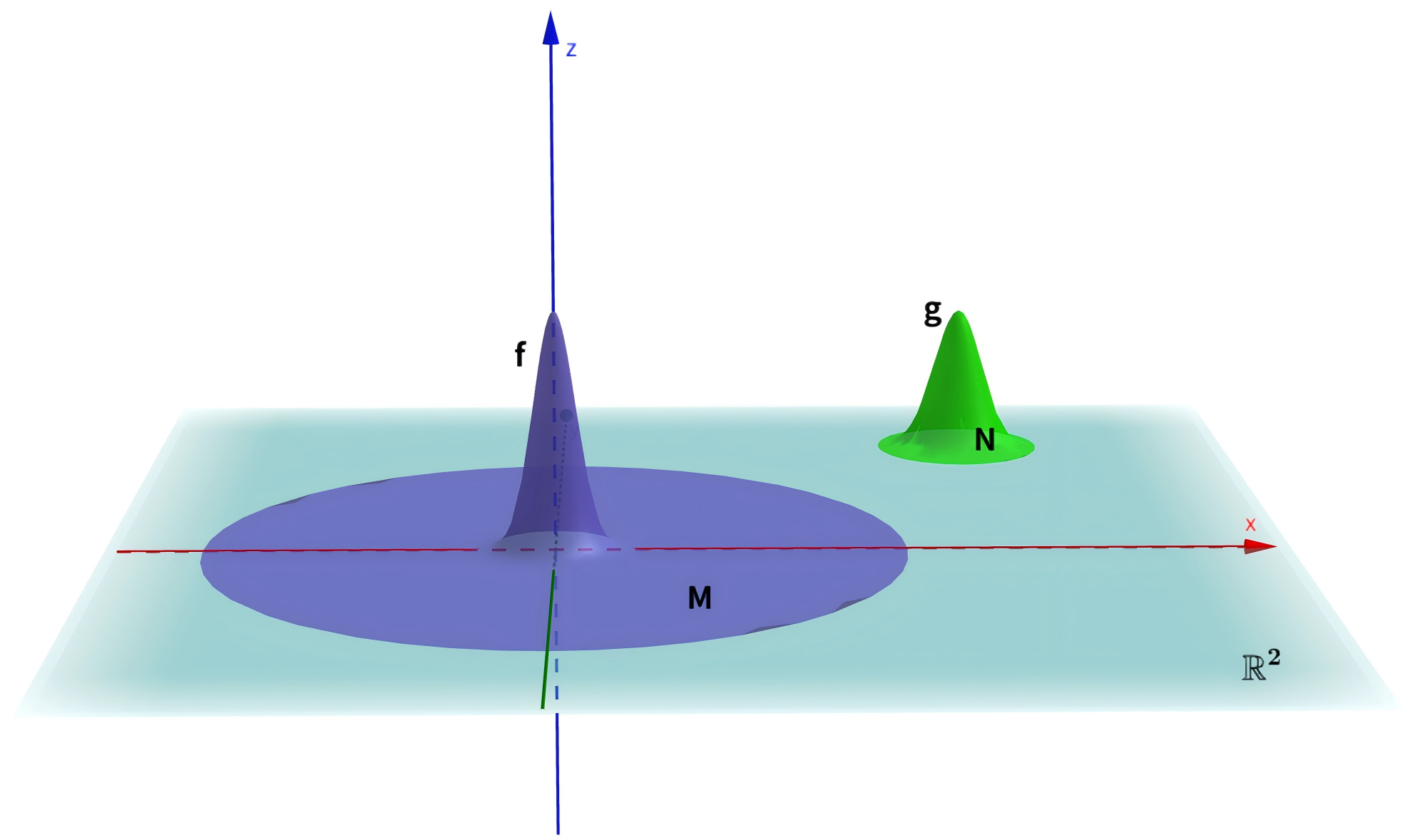}
			\caption{The preimages of the purple and green regions correspond to the manifolds $\mathcal{M}$ and $\mathcal{N}$, respectively.}
			\label{essential_1_para}
		\end{figure}
		Let $\mathcal{M}_i =\{x\in \mathbb{R}^n; f(x)\leq i\}$, $\mathcal{N}_j = \{x\in \mathbb{R}^n; g(x)\leq j\}$, where $i,j\in\mathbb{R}$. Since $\mathbb{R}^n = \mathcal{M}_i\cup \mathcal{N}_j$, we can get the following exact sequence, i.e., the Mayer-Vietoris sequence,
		\begin{align*}
			\cdots\rightarrow H_{k}(\mathcal{M}_i\cap \mathcal{N}_j)\xrightarrow{\Phi}H_k(\mathcal{M}_i)\oplus H_k(\mathcal{N}_j)\xrightarrow{\Psi}H_k(\mathbb{R}^n)\xrightarrow{\partial}H_{k-1}(\mathcal{M}_i\cap \mathcal{N}_j)\\
			\rightarrow \cdots  \rightarrow H_0(\mathbb{R}^n)\rightarrow 0.
		\end{align*}
		When $k\geq 1$, $H_k(\mathbb{R}^n)=0$. Then we have $H_k(\mathcal{M}_i\cap \mathcal{N}_j)\cong H_k(\mathcal{M}_i)\oplus H_k(\mathcal{N}_j)$ for $k\geq 1$.
		Define the bi-parameter persistence module $(M,\pi)$ by setting $M_{i,j}(f,g)=H_{*}(\mathcal{M}_i\cap \mathcal{N}_j)$, and define the single-parameter persistence module $(M(f,g)|_{\vec{l}},\pi(f,g)|_{\vec{l}})$ by $$(M(f,g)|_{\vec{l}})_t:=M_{at+b},\  (\pi(f,g)|_{\vec{l}})_{s,t}:=\pi(f,g)|_{as+b,at+b},$$ where $l\in\mathbb{R}^n$ is a line parametrized by $t\in\mathbb{R}\mapsto at+b$ with $a\in (\mathbb{R_{+}})^n\setminus\{0\}$.
		Now we define single-parameter persistence modules $(M(f),\pi(f))$ and $(M(g),\pi(g))$ by $M_{i}(g)=H_{*}(\mathcal{M}_i)$ and $M_{j}(f)=H_{*}(\mathcal{N}_j)$, respectively. 
		
		Let $\vec{u}=(u_1,u_2)$, $\vec{v}=(v_1,v_2)$. There exist $\vec{l'}=a't+b'$, $s'$ and $t'$ such that $(u_1,u_2)=a's'+b'$, $(v_1,v_2)=a't'+b'$, then we have the following commutative diagram,
		\begin{center}
			\begin{tikzcd}[column sep=huge]
				(M(f,g)|_{\vec{l'}})_{s'} \arrow{d}{\cong} \arrow{rr}{\pi(f,g)_{s',t'}} && (M(f,g)|_{\vec{l'}})_{t'} \arrow{d}{\cong}\\
				{M(f)_{u_1}\oplus M(g)_{u_2}} \arrow{rr}{\pi(f)_{u_1,v_1}\oplus \pi(g)_{v_1,v_2} }        && {M(f)_{u_2}\oplus V(g)_{v_2}}
			\end{tikzcd}
		\end{center}
		Then we have $\mathbb{M}(f,g)|_{\vec{l'}}\cong \mathbb{M}(f)\oplus \mathbb{M}(g)$. 
		
		In this case, the bi-parameter persistence module decomposes exactly into the direct sum of two single-parameter persistence modules in each direction.  Consequently, performing classification using this bi-parameter filtration yields no fundamental improvement or difference compared to direct classification using two independent single-parameter persistent homology.
	\end{example}
	\item[(3)] Driven by motivation (2), the two selected filter functions should be defined such that the intersection of their sublevel sets is non-empty; therefore, we apply Gaussian smoothing to both the voxel values and the Laplacian operator (Laplacian of Gaussian, LoG).
\end{enumerate}

\subsection{Contributions}

In this paper, we provide a framework to build bi-parameter filtrations on volumes. 
\begin{enumerate}
	\item[$\bullet$] G-LoG Bifiltration \& Stability: We define the G-LoG bi-filtration and we prove the interleaving distance on the persistence modules obtained from our bi-filtration on the bounded functions is stable with respect to the maximum norm of the bounded functions.
	
	\item[$\bullet$] Experimental Validation: We conduct experiments on the MedMNIST dataset to assess the performance of our bi-filtration in data-rich environments. Our key findings include, 
	\begin{enumerate}
		\item[$\bullet$] Superiority over single-parameter persistent homology: Our results demonstrate that our bi-filtration outperforms single-parameter persistent homology.
		
		\item[$\bullet$]  Performance in 2D Image Classification: We achieve results competitive with several established baseline methods.
		
		\item[$\bullet$] Effectiveness in 3D Image Classification: For 3D image classification, our method yields competitive performance compared to leading baseline approaches.
	\end{enumerate}
\end{enumerate}

To foster further developments at the intersection of multi-parameter persistent homology, we release our source code under:\url{https://github.com/HeJiaxing-hjx/G-LoG-bifiltration-for-medical-imaging-classification.git}.

\section{Preliminaries}

In this section, we will introduce some definitions and properties used in this paper.

Let $\mathbb{R}$ be the set of real numbers. For vectors $\boldsymbol{s}=(s_{1},\cdots,s_{m})$, $\boldsymbol{t}=(t_{1},\cdots,t_{m})$ in $\mathbb{R}^{m}$, there is a natural partial order on $\mathbb{R}^{m}$ by taking $\boldsymbol{s}\leq_m \boldsymbol{t}$ if and only if $s_{i} \leq t_{i}$ for all $1 \leq i \leq m$. Consider the sublevel set filtration $$
X^{\vec{\gamma}}_{\boldsymbol{s}}:=\{\ x\in X\ |\ \vec{\gamma}(x)\leq_{m}  \boldsymbol{s}\ \}
$$ with natural inclusion $\iota_{\boldsymbol{s},\boldsymbol{t}}$, where $\vec{\gamma}=(\gamma^1,\cdots,\gamma^m)$ which we called filter function. 

We denote $M_{\boldsymbol{s}}=H_{*}(X_{\boldsymbol{s}}; \mathbb{F})$. In this paper, we only consider $X=\mathbb{R}^n$.
\begin{definition}
	A persistence module $(M,\pi)$ is a family of $\mathbb{F}$-modules $\{M_{\boldsymbol{s}}\}_{\boldsymbol{s}\in \mathbb{R}^m}$ together with homomorphisms $\{\pi_{\boldsymbol{s},\boldsymbol{t}}:M_{\boldsymbol{s}}\rightarrow M_{\boldsymbol{t}}\}_{\boldsymbol{s},\boldsymbol{t}\in \mathbb{R}^m, \boldsymbol{s}\leq_m \boldsymbol{t}}$  such that  $\pi_{\boldsymbol{s},\boldsymbol{u}}=\pi_{\boldsymbol{t},\boldsymbol{u}}\circ\pi_{\boldsymbol{s},\boldsymbol{t}}$ and $\pi_{\boldsymbol{s},\boldsymbol{s}}=\text{id}$ for any $\boldsymbol{s}\leq_m\boldsymbol{t}\leq_{m}\boldsymbol{u}$.
\end{definition}
We denote $(M,\pi)$ by $\mathbb{M}$ in brief and we also call $\{\pi_{\boldsymbol{s},\boldsymbol{t}}\}$ transition maps of $\mathbb{M}$. Let $(M,\pi)$, $(M',\pi')$ be two persistence modules.
\begin{definition}
	A persistence morphism $\varphi:(M,\pi)\rightarrow(M',\pi')$ is a family of linear maps $h_{\boldsymbol{s}}:M_{\boldsymbol{s}} \rightarrow M'_{ \boldsymbol{s}}$ such that the following diagram commutes for all $\boldsymbol{s}\leq_m  \boldsymbol{t}$,
	\begin{center}
		\begin{tikzcd}
			M_{\boldsymbol{s}} \arrow{d}{h_{\boldsymbol{s}}} \arrow{r}{\pi_{\boldsymbol{s},\boldsymbol{t}}} & M_{ \boldsymbol{t}} \arrow{d}{h_{\boldsymbol{t}}}\\
			M'_{\boldsymbol{s}} \arrow{r}{\pi'_{\boldsymbol{s},\boldsymbol{t}}}         & M'_{\boldsymbol{t}}
		\end{tikzcd}
	\end{center}
\end{definition}

Two persistence modules $\mathbb{M}=(M,\pi)$ and  $\mathbb{M}'=(M',\pi')$ are isomorphic if there exist two morphisms $h_1:\mathbb{M}\rightarrow \mathbb{M}'$ and $h_2:\mathbb{M}'\rightarrow \mathbb{M}$ such that both compositions $h_1\circ h_2$ and $h_2\circ h_1$ are the identity morphisms on the corresponding persistence modules, where the identity morphism on $\mathbb{M}$ is the identity on $M_{\boldsymbol{t}}$ for all $\boldsymbol{t}$.

\begin{definition}
	Let $(M,\pi)$, $(M',\pi')$ be two persistence modules. Their direct sum $(N,\theta)$ is the persistence module whose underlying modules are $N_{\boldsymbol{t}} = M_{\boldsymbol{t}}\oplus M'_{\boldsymbol{t}}$ and accordingly, $\theta_{\boldsymbol{s},\boldsymbol{t}}=\pi_{\boldsymbol{s},\boldsymbol{t}}\oplus\pi'_{\boldsymbol{s},\boldsymbol{t}}$.
\end{definition}

For a persistence module $(M,\pi)$ and $\boldsymbol{\epsilon}\in\mathbb{R}^m$, define a persistence module  $(M[\boldsymbol{\epsilon}],\pi[\boldsymbol{\epsilon}])$ by taking $(M[\boldsymbol{\epsilon}])_{\boldsymbol{s}}=M_{\boldsymbol{s}+\boldsymbol{\epsilon}}$ and $(\pi[\boldsymbol{\epsilon}])_{\boldsymbol{s},\boldsymbol{t}}=\pi_{\boldsymbol{s}+\boldsymbol{\epsilon},\boldsymbol{t}+\boldsymbol{\epsilon}}$. This new persistence module is called the $\boldsymbol{\epsilon}$-shift of $\mathbb{M}$. The map $\Phi^{\boldsymbol{\epsilon}}:(M,\pi)\rightarrow(M[\boldsymbol{\epsilon}],\pi[\boldsymbol{\epsilon}])$ defined by $\Phi^{\boldsymbol{\epsilon}}_{\boldsymbol{t}}=\pi_{\boldsymbol{t},\boldsymbol{t}+\boldsymbol{\epsilon}}$ is an $\boldsymbol{\epsilon}$-shift morphism of persistence modules.

Now let us introduce interleaving distance on the space of persistence modules. If $h$ is a homomorphism from $\mathbb{M}$ to $\mathbb{M}'$, then  $h(\boldsymbol{\epsilon}): \mathbb{M}(\boldsymbol{\epsilon})\rightarrow \mathbb{M'}(\boldsymbol{\epsilon})$.

\begin{definition}
	Given $\epsilon>0$, we say that two persistence modules $\mathbb{M}$ and $\mathbb{M}'$ are $\boldsymbol{\epsilon}$-interleaved if there exist two morphisms $h_1:\mathbb{M}\rightarrow \mathbb{M}'(\boldsymbol{\epsilon})$ and $h_2: \mathbb{M}'\rightarrow \mathbb{M}(\boldsymbol{\epsilon})$, such that $h_2(\boldsymbol{\epsilon})\circ h_1={\Phi}_{\mathbb{M}}^{2\boldsymbol{\epsilon}}$ and ${\Phi}_{\mathbb{M}'}^{2\boldsymbol{\epsilon}}=h_1(\boldsymbol{\epsilon})\circ h_2$, where $\mathbb{M}(\boldsymbol{\epsilon})$ is the shift module $\{M_{\boldsymbol{s}+\boldsymbol{\epsilon}}\}_{s\in\mathbb{R}^m}$, $\boldsymbol{\epsilon}=(\epsilon,\cdots,\epsilon)\in\mathbb{R}^m$, ${\Phi}_{\mathbb{M}}^{2\boldsymbol{\epsilon}}$ and ${\Phi}_{\mathbb{M}'}^{2\boldsymbol{\epsilon}}$ are the shift morphisms.
	
	The interleaving distance between two multi-parameter persistence modules $\mathbb{M}$ and $\mathbb{M}'$ is defined to be
	$$
	d_{I}(\mathbb{M},\mathbb{M}')=\text{inf}\{\epsilon>0\ |\ \mathbb{M}\  \text{and} \ \mathbb{M}' \  \text{are} \  \epsilon\text{-interleaved}\}.
	$$
\end{definition}

The main property of $d_I$ is that it is stable for multi-parameter filtrations that are obtained from the sublevel sets of functions. More precisely, given two continuous function $\vec{\gamma}_1$,$\vec{\gamma}_2:X\rightarrow\mathbb{R}^m$, we denote $\mathbb{M}(\vec{\gamma}_1)$ and $\mathbb{M}(\vec{\gamma}_2)$ by the multi-parameter persistence modules obtained from the corresponding filtrations $X^{\vec{\gamma}_1}$ and $X^{\vec{\gamma}_2}$, then
\begin{align*}
	d_{I}(\mathbb{M}(\vec{\gamma}_1),\mathbb{M}(\vec{\gamma}_2))\leq \|\vec{\gamma}_1-\vec{\gamma}_2\|_{\infty},
\end{align*}
where $$
{\Vert}\vec{\gamma}{\Vert}_{\infty}=\left\{
\begin{array}{rcl}
	\sup_{p\in{X}}{\Vert}\vec{\gamma}(p){\Vert}_{\infty}= \sup_{p\in{X}}\max\{|\gamma^1(p)|,\cdots,|\gamma^m(p)|\}    &   &   {\text{if}\   X\   \neq\    \emptyset},\\
	0  \quad \quad \      &   &   {\text{if}\   X\   =\    \emptyset}.     
\end{array}
\right.
$$

\begin{theorem} [\cite{Micheal-2015}] \label{stable_di}
	$d_{I}$ is stable.
\end{theorem}

\section{G-LoG Bifiltration and its stability}

One reason for the popularity of the multi-parameter persistence modules in TDA is that the transformation of a data set to the multi-parameter persistence modules is stable (Lipschitz continuous) with respect to the interleaving distances. In this section, we give the definition of our G-LoG bi-filtration and prove its stability.

To construct bi-filtration from volumes and images, we adopt notation consistent with \cite{Bergo-2019}. We represent data as the function space $\Phi$, defined as a set of real-valued functions $\{\varphi_{i}\}_{i}$ mapping from a topological space $\mathbb{R}^n$ to $\mathbb{R}$, we have $\|\varphi_1-\varphi_2\|_{\infty}=\sup_{x\in\mathbb{R}^n}|\varphi_1(x)-\varphi_2(x)|$. 

Then we can give our definition of G-LoG bi-filtration. 

Let $x=(x_1,\cdots,x_n)\in \mathbb{R}^n$, and let $G$ be the Gaussian kernel defined on $\mathbb{R}^n$, 
\begin{align}\label{kerg}
	G(x_1,\cdots,x_n)=\exp\{\frac{-\sum_{i=1}^{n}x^{2}_{i}}{2\sigma^2}\},
\end{align} we have
\begin{align*}
	\triangle G(x_1,\cdots, x_n)=\sum_{i=1}^n\frac{\partial^2}{\partial^{2}x_{i}}G(x_1,\cdots,x_n)=\frac{\sum_{i=1}^{n}x^{2}_{i}-n\sigma^2}{\sigma^4}\exp\{\frac{-\sum_{i=1}^{n}x^{2}_{i}}{2\sigma^2}\}. 
\end{align*}
\begin{definition}
	Let $\varphi$ be a continous function from $\mathbb{R}^{n}$ to $\mathbb{R}$. We define  G-LoG bi-filtration on $\varphi$ by $\vec{\gamma}_{\psi}=(\gamma^{1}_{\varphi}, \gamma^{2}_{\varphi})$,  where
	\begin{align*}
		&\gamma^{1}_{\varphi}(x) = \int_{\mathbb{R}}\cdots\int_{\mathbb{R}}\varphi(x_1-\alpha_1,\cdots,x_n-\alpha_n)\cdot G(\alpha_1,\cdots,\alpha_n)d\alpha_1\cdots d\alpha_n,\\
		&\gamma^{2}_{\varphi}(x) = \int_{\mathbb{R}}\cdots\int_{\mathbb{R}}\varphi(x_1-\alpha_1,\cdots,x_n-\alpha_n)\cdot \triangle G(\alpha_1,\cdots,\alpha_n)d\alpha_1\cdots d\alpha_n.
	\end{align*}
\end{definition}
Then we can get the following stability of the interleaving distance on $\vec{\gamma}_{\varphi_1}$ and $\vec{\gamma}_{\varphi_2}$ with respect to the maximum norm on $\varphi_1$ and $\varphi_2$.

\begin{theorem}
	Let $\varphi_1$, $\varphi_2: \mathbb{R}^n\rightarrow \mathbb{R}$ be two continuous functions  and let $\mathbb{M}(\vec{\gamma}_{\varphi_1})$ and $\mathbb{M}(\vec{\gamma}_{\varphi_2})$ be the corresponding multi-parameter persistence modules induced by $\vec{\gamma}_{\varphi_1} = (\gamma^{1}_{\varphi_1},\gamma^{2}_{\varphi_1})$ and $\vec{\gamma}_{\varphi_2} = (\gamma^{1}_{\varphi_2},\gamma^{2}_{\varphi_2})$, respectively. Then, the following stability inequality holds,
	\begin{align*}
		d_I(\mathbb{M}(\vec{\gamma}_{\varphi_1}),\mathbb{M}(\vec{\gamma}_{\varphi_2}))\leq  \max((2\pi \sigma^2)^{\frac{n}{2}},\frac{2n(2\pi \sigma^2)^{\frac{n}{2}}}{\sigma^2})\cdot \|\varphi_1-\varphi_2\|_{\infty}
	\end{align*}
\end{theorem}

\begin{proof}
	Recall that the Gaussian integral yields
	\begin{align}\label{4.2}
		\int_{\mathbb{R}^n} \exp\{-\frac{\sum_{i=1}^{n}\alpha^{2}_{i}}{2\sigma^2}\} d\alpha_1 \cdots d\alpha_n = (2\pi \sigma^2)^{\frac{n}{2}}.
	\end{align}
	From this, we obtain
	\begin{align*}
		\|\gamma^{1}_{\varphi_1}-\gamma^{1}_{\varphi_2}\|_{\infty} &\leq \|\varphi_1-\varphi_2\|_{\infty} \int_{\mathbb{R}^n} |G(\alpha_1,\dots,\alpha_n)| d\alpha_1 \cdots d\alpha_n \\
		&= (2\pi \sigma^2)^{\frac{n}{2}} \|\varphi_1-\varphi_2\|_{\infty}.
	\end{align*}
	Similarly, for the second component of the bi-filtration, we have
	\begin{align*}
		&\|\gamma^{2}_{\varphi_1}-\gamma^{2}_{\varphi_2}\|_{\infty} \\
		&\leq \|\varphi_1-\varphi_2\|_{\infty} \int_{\mathbb{R}^n} \frac{1}{\sigma^4} \left( \left| \sum_{i=1}^{n}\alpha^{2}_{i} G(\alpha) \right| + \left| n\sigma^2 G(\alpha) \right| \right) d\alpha_1 \cdots d\alpha_n.
	\end{align*}
	From (\ref{4.2}), we have $\alpha_i\sim N(0,\sigma^2)$, and notice that the second moment is $\mathbb{E}[\sum^n_{i=1}\alpha^2_i] = n\sigma^2$, it follows that
	\begin{align*}
		&\frac{1}{\sigma^4} \left( \int_{\mathbb{R}^n} \left| \sum^n_{i=1}\alpha^2_i  \exp\{-\frac{\sum \alpha^2_i}{2\sigma^2}\}\right| d\alpha_1\cdots d\alpha_n + \int_{\mathbb{R}^n} n\sigma^{2} \exp\{-\frac{\sum \alpha^2_i}{2\sigma^2}\} d\alpha_1\cdots d\alpha_n \right) \\
		&= \frac{1}{\sigma^4} \left( n\sigma^{2} \cdot (2\pi \sigma^2)^{\frac{n}{2}} + n\sigma^2 \cdot (2\pi \sigma^2)^{\frac{n}{2}} \right) = \frac{2n(2\pi \sigma^2)^{\frac{n}{2}}}{\sigma^2}.
	\end{align*}
	Consequently, we establish the bound for the combined function $\vec{\gamma}_{\varphi}$:
	\begin{align*}
		\|\vec{\gamma}_{\varphi_1}-\vec{\gamma}_{\varphi_2}\|_{\infty} \leq \max \left\{ (2\pi \sigma^2)^{\frac{n}{2}}, \frac{2n(2\pi \sigma^2)^{\frac{n}{2}}}{\sigma^2} \right\} \cdot \|\varphi_1-\varphi_2\|_{\infty}.
	\end{align*}
	Finally, by Theorem \ref{stable_di} regarding the stability of the interleaving distance, we conclude,
	\begin{align*}
		d_I(\mathbb{M}(\vec{\gamma}_{\varphi_1}),\mathbb{M}(\vec{\gamma}_{\varphi_2})) &\leq \|\vec{\gamma}_{\varphi_1}-\vec{\gamma}_{\varphi_2}\|_{\infty}=\sup_{x\in\mathbb{R}^n}\|\vec{\gamma}_{\varphi_1}-\vec{\gamma}_{\varphi_2}\|_{\infty}\\
		&=\max\{\|\gamma^{1}_{\varphi_1}-\gamma^{1}_{\varphi_2}\|_{\infty},\|\gamma^{2}_{\varphi_1}-\gamma^{2}_{\varphi_2}\|_{\infty}\} \leq C \cdot \|\varphi_1-\varphi_2\|_{\infty},
	\end{align*}
	where $C = \max\left\{(2\pi \sigma^2)^{\frac{n}{2}}, \frac{2n(2\pi \sigma^2)^{\frac{n}{2}}}{\sigma^2}\right\}$.
\end{proof}

\begin{remark}
	To build bi-filtrations quickly, we use the approximate multi-parameter persistence modules $\tilde{M}_{\delta}^{MMA}$ defined in \cite{LCB-2025}, where $\delta>0$. They proved that for two continous functions $f$, $g: X\rightarrow \mathbb{R}^m$, the following stability holds, \begin{align*}
		d_I(\tilde{M}_{\delta}^{MMA}(f),\tilde{M}_{\delta}^{MMA}(g))\leq \|f-g\|_{\infty}+\delta,
	\end{align*} Let $f$ and $g$ be $\vec{\gamma}_{\varphi_1}$ and $\vec{\gamma}_{\varphi_2}$ respectively, we have the following stability,
	\begin{align*}
		&d_I(\tilde{M}_{\delta}^{MMA}(\vec{\gamma}_{\varphi_1}),\tilde{M}_{\delta}^{MMA}(\vec{\gamma}_{\varphi_2}))\\
		&\leq \max\{(2\pi \sigma^2)^{\frac{n}{2}},\frac{2n(2\pi \sigma^2)^{\frac{n}{2}}}{\sigma^2}\}\cdot \|\varphi_1-\varphi_2\|_{\infty}+\delta.
	\end{align*}
\end{remark}

\section{Experiments}

We conduct experiments on the MedMNIST dataset to evaluate the effectiveness of our G-LoG bi-filtration. We compare our results against all baseline machine learning models reported in MedMNIST \cite{YSWLZKPN-2023}, as well as the results from the single-parameter Topo-Med  approach \cite{NKKC-2025}. The bi-parameter cubical filtrations are implemented using the multipers library (\url{https://davidlapous.github.io/multipers/})  and the GUDHI library (\url{https://gudhi.inria.fr/}). We propose the following pipeline for feature generation by G-LoG bi-filtration (illustrated in Figure \ref{pipeline}).
\begin{figure}[ht]
	\centering
	\begin{tikzpicture}[
		item/.style={rectangle, draw=black, thick, minimum width=1.5cm, minimum height=0.8cm, font=\scriptsize\sffamily, align=center},
		arrow/.style={-{Stealth[scale=1.0]}, thick},
		label/.style={font=\tiny\itshape, color=black!80}
		]
		
		\node[item] (raw) {2D/3D\\ Medical \\ Image};
		\node[item, right=1.5cm of raw] (glog) {bi-filtered\\ simplicial\\ complexes};
		\node[item, right=1.5cm of glog] (tda) {bi-parameter \\ persistence\\ modules};
		\node[item, right=1.5cm of tda] (mlp) {bi-parameter\\
			persistence\\
			image};
		
		\draw[arrow] (raw) -- node[above, label] {G-LoG}
		node[below, label] {bi-filtration}(glog);
		
		\draw[arrow] (glog) -- node[above, label] {Homology}(tda); 
		
		\draw[arrow] (tda) -- node[above, label] {Vectorization} (mlp);
		
		\node[right=0.8cm of mlp, font=\footnotesize\sffamily] (out) {Result};
		\draw[arrow] (mlp) -- node[above, label] {MLP}(out);
		
	\end{tikzpicture}
	\caption{Classification pipeline using G-LoG bi-filtration}
	\label{pipeline}
\end{figure}
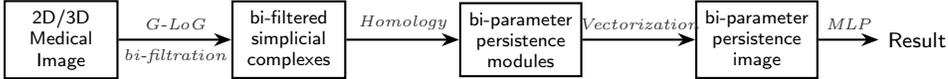

Our bi-filtrations are generated on a laptop equipped with an AMD Ryzen 7 5800H (Radeon Graphics) and 16GB of RAM. The MLP models are subsequently trained on a system featuring an Intel Core Ultra 9 185H (2.30 GHz) and 32GB of RAM.

{\bf Dataset} MedMNIST is a large-scale, standardized collection of biomedical images inspired by the MNIST format, comprising 12 2D datasets and 6 3D datasets. All images are preprocessed into a uniform resolution of $28 \times 28$ (for 2D) or $28 \times 28 \times 28$ (for 3D) and are provided with corresponding classification labels, making them accessible to users without prior domain expertise. Covering primary data modalities in biomedical images, MedMNIST is designed to perform classification on lightweight 2D and 3D images with various data scales (from 100 to 100,000) and diverse tasks (binary/multi-class, ordinal regression and multi-label).  With approximately 708,000 2D images and 10,000 3D samples in total, this dataset supports a wide range of research and educational needs in biomedical image analysis, computer vision, and machine learning.

{\bf Construction of Bi-parameter Filtrations} First, we convert the 2D color images into grayscale and normalize the pixel values from the range $[0, 255]$ to $[0, 1]$. Next, we construct the bi-parameter filtration functions. As established in our discussion on research motivation (2), if the intersection of the preimages of two sublevel set functions is too small, the extracted topological features degenerate into those obtained by applying two separate single-parameter filtrations, leading to a  reduction in corresponding feature count. To intuitively illustrate the importance of capturing this intersection, we employ Gaussian and Gaussian-Laplacian convolutions as the two sublevel set filtrations on medical images. The parameters for the Gaussian kernels are set to $\sigma=0$ (representing no convolution), $0.5$, $1$ and $1.5$, while the $\sigma$ for the Gaussian-Laplacian (LoG) kernel is fixed at $1$. Finally, we construct the bi-parameter persistence modules. While the GUDHI library is highly efficient for generating single-parameter cubical complexes, the multipers library \cite{multipers-2024} provides a user-friendly and high-performance framework for multi-parameter persistent homology, enabling the generation of approximate persistence modules and integrating various vectorization methods. By leveraging both GUDHI  \cite{gudhi-2014} and multipers, we are able to construct approximate bi-parameter persistence modules.

{\bf Running time} It takes about $0.1$ second to generate a bi-parameter persistence module from a $28 \times 28$ medical image, whereas a $28 \times 28 \times 28$ medical volume requires approximately 90 seconds.


{\bf Vectorization} We leverage the multipers library to extract vectorized features from the multi-parameter persistence modules. Specifically, we generate Multi-parameter Persistence Images (MPIs) \cite{Math-2020} by using a Gaussian kernel with a bandwidth of $0.01$ and a weight parameter $p=2$. For each $i$-dimensional persistent homology $H_i$, the resulting persistence image has a resolution of $50 \times 50$. For 2D image data, we generate MPIs for $H_0$ and $H_1$; after concatenation, each 2D image is represented by a $5000$-dimensional vector. For 3D volumetric data, we generate MPIs for $H_0$, $H_1$ and $H_2$. By concatenating these images, we obtain a $7500$-dimensional vector for each 3D volume.

{\bf Network Hyperparameters} Consistent with the methodology in \cite{NKKC-2025}, we employ a Multi-Layer Perceptron (MLP) with three hidden layers to train on the multi-parameter persistence images. The MLP architecture consists of an input layer, three hidden layers, and an output layer. Specifically, the three hidden layers are: a fully connected layer with 256 neurons and ReLU activation, followed by a fully connected layer with 128 neurons and ReLU activation, and a final hidden layer with 64 neurons and ReLU activation. The output layer utilizes the Softmax activation function. For model optimization, we use the cross-entropy loss function and the Adam optimizer with a learning rate of $0.001$. The training is conducted for 100 epochs with a batch size of 32. Additionally, an early stopping strategy is implemented with a patience of 20, and the model achieving the highest AUC score on the validation set is saved for final evaluation.

{\bf Evaluation Metrics} Consistent with prior research, we employ the average Area Under the Curve (AUC) and average Accuracy (ACC) as our evaluation metrics. AUC is a threshold-free metric to evaluate the continuous prediction scores, while ACC evaluates the discrete prediction labels given threshold (or argmax). The computational procedures for these metrics are identical to those described in \cite{YSWLZKPN-2023}.

\subsection{Results} As illustrated in Table \ref{table_2d_1}, Table \ref{table_2d_2} and Table \ref{table_3d}, we compare the performance of our proposed method against the baseline results reported in \cite{YSWLZKPN-2023} and \cite{NKKC-2025} on the MedMNIST dataset.
\begin{table}[H]\setlength{\tabcolsep}{1.4mm}
	\fontsize{6.5pt}{10.5pt}\selectfont
	\centering
	\begin{tabular}{|c|c|c|c|c|c|c|c|c|c|c|c|c|c|}
		\hline
		\multicolumn{2}{|c}{\multirow{1}{*}{Method}}
		&\multicolumn{2}{|c|}{PathMNIST}&\multicolumn{2}{c|}{ChestMNIST}
		&\multicolumn{2}{c|}{DermaMNIST}&\multicolumn{2}{c|}{OCTMNIST}&\multicolumn{2}{c|}{PneumoniaMNIST}
		&\multicolumn{2}{c|}{RetinaMNIST}\\
		\cline{3-14} 	
		\multicolumn{2}{|c}{} & \multicolumn{1}{|c|}{AUC} & ACC &  AUC & ACC &  AUC & ACC  &  AUC & ACC  &  AUC & ACC  &  AUC & ACC \\
		\cline{1-14}
		\multicolumn{2}{|c|}{ResNet-18 (28)}  & {0.983} & {0.907} &{0.768} & {0.947} & {0.917} &{0.735}  & {0.943} & {0.743} &{0.944} & {0.854} & {0.717} &{0.524}\\
		\cline{1-14}	
		\multicolumn{2}{|c|}{ResNet-18 (224)}  & {0.989} & {0.909} &{0.773} & {0.947} & \textbf{0.920} &{0.754} & {0.958} & {0.763} &{0.956} & {0.864} & {0.710} &{0.493}\\
		\cline{1-14}
		\multicolumn{2}{|c|}{ResNet-50 (28)}  & \textbf{0.990} & \textbf{0.911} &{0.769} & {0.947} & {0.913} &{0.735}& {0.952} & {0.762} &{0.948} & {0.854} & {0.726} &{0.528}\\
		\cline{1-14}	
		\multicolumn{2}{|c|}{ResNet-50 (224)}  & {0.989} & {0.892} &{0.773} & \textbf{0.948} & {0.912} &{0.731} & {0.958} & \textbf{0.776} &{0.962} & {0.884} & {0.716} &{0.511}\\
		\cline{1-14}
		\multicolumn{2}{|c|}{auto-sklearn}  & {0.934} & {0.716} &{0.649} & {0.779} & {0.902} &{0.719} &{0.887} & {0.601} & {0.942} &{0.855} & {0.690} & {0.515}\\
		\cline{1-14}
		\multicolumn{2}{|c|}{AutoKeras} & {0.959} & {0.834} &{0.742} & {0.937} & {0.915} &{0.749} & {0.955} & {0.763} &{0.947} & {0.878} & {0.719} &{0.503}\\
		\cline{1-14}
		\multicolumn{2}{|c|}{Google AutoML Vision}  & {0.944} & {0.728} &{0.778} & \textbf{0.948} & {0.914} &\textbf{0.768}& \textbf{0.963} & {0.771} &\textbf{0.991} & \textbf{0.946} & \textbf{0.750} &\textbf{0.531}\\
		\cline{1-14}
		\multicolumn{2}{|c|}{Topo-Med (MLP)}  & {0.942} & {0.683} &\textbf{0.787} & {0.530} & {0.904} &{0.669}& {0.710} & {0.450} &{0.845} & {0.762} & {0.728} &{0.458}  \\
		\cline{1-14}
		\multirow{4}{*}{\makecell{\\ Ours\\(MLP)\\}}  & {$\sigma=0$} & {0.955} & {0.753} &{0.606} & {0.947} & {0.809} &{0.703} & {0.859} & {0.522} &{0.907} & {0.817} & {0.611} &{0.518}\\
		\cline{2-14}
		& {$\sigma=0.5$} & {0.954} & {0.753} &{0.608} & {0.947} & {0.806} &{0.708}& {0.872} & {0.567} &{0.910} & {0.825} & {0.624} &{0.500}\\
		\cline{2-14}
		& {$\sigma=1$}& {0.940} & {0.708} &{0.610} & {0.947} & {0.815} &{0.700}& {0.869} & {0.548} &{0.906} & {0.825} & {0.643} &{0.498}\\
		\cline{2-14}
		& {$\sigma=1.5$}& {0.939} & {0.702} &{0.603} & {0.947} & {0.816} &{0.698} & {0.854} & {0.529} &{0.891} & {0.795} & {0.638} &{0.483}\\
		\hline
	\end{tabular}
	\caption{Comparison of seven deep learning baselines, single-parameter and our G-LoG bi-filtration on MedMNIST2D (\uppercase\expandafter{\romannumeral1}).}	
	\label{table_2d_1}
\end{table}

\begin{table}[H]\setlength{\tabcolsep}{1.5mm}
	\fontsize{6.5pt}{10.5pt}\selectfont
	\centering
	\begin{tabular}{|c|c|c|c|c|c|c|c|c|c|c|c|c|c|}
		\hline
		\multicolumn{2}{|c}{\multirow{1}{*}{Method}}
		&\multicolumn{2}{|c|}{BreastMNIST}&\multicolumn{2}{c|}{BloodMNIST}
		&\multicolumn{2}{c|}{TissueMNIST}&\multicolumn{2}{c|}{OrganAMNIST}&\multicolumn{2}{c|}{OrganCMNIST}
		&\multicolumn{2}{c|}{OrganSMNIST}\\
		\cline{3-14} 	
		\multicolumn{2}{|c}{} & \multicolumn{1}{|c|}{AUC} & ACC &  AUC & ACC &  AUC & ACC &  AUC & ACC&  AUC & ACC&  AUC & ACC\\
		\cline{1-14}
		\multicolumn{2}{|c|}{ResNet-18 (28)}  & {0.901} & \textbf{0.863} &\textbf{0.998} & {0.958} & {0.930} &{0.676} & {0.997} & {0.935} &{0.992} & {0.900} & {0.972} &{0.782}\\
		\cline{1-14}	
		\multicolumn{2}{|c|}{ResNet-18 (224)}  & {0.891} & {0.833} &\textbf{0.998} & {0.963} & {0.933} &{0.681}& \textbf{0.998} & \textbf{0.951} &\textbf{0.994} & \textbf{0.920} & {0.974} &{0.778}\\
		\cline{1-14}
		\multicolumn{2}{|c|}{ResNet-50 (28)}  & {0.857} & {0.812} &{0.997} & {0.956} & {0.931} &{0.680} & {0.997} & {0.935} &{0.992} & {0.905} & {0.972} &{0.770}\\
		\cline{1-14}	
		\multicolumn{2}{|c|}{ResNet-50 (224)}  & {0.866} & {0.842} &{0.997} & {0.956} & {0.931} &{0.680}  & \textbf{0.998} & {0.947} &{0.993} & {0.911} & \textbf{0.975} &{0.785}\\
		\cline{1-14}
		\multicolumn{2}{|c|}{auto-sklearn}  & {0.836} & {0.803} &{0.984} & {0.878} & {0.828} &{0.532}& {0.963}  & {0.762}  &{0.976} & {0.829} & {0.945} &{0.672}\\
		\cline{1-14}
		\multicolumn{2}{|c|}{AutoKeras} & {0.871} & {0.831} &\textbf{0.998} & {0.961} & \textbf{0.941} &\textbf{0.703}& {0.994} & {0.905} &{0.990} & {0.879} & {0.974} &\textbf{0.813}\\
		\cline{1-14}
		\multicolumn{2}{|c|}{Google AutoML Vision}  & \textbf{0.919} & {0.861} &\textbf{0.998} & \textbf{0.966} & {0.924} &{0.673}& {0.990} & {0.886} &{0.988} & {0.877} & {0.964} &{0.749}\\
		\cline{1-14}
		\multicolumn{2}{|c|}{Topo-Med (MLP)}  & {0.821} & {0.737} &{0.973} & {0.798} & {0.837} &{0.450} & {0.921} & {0.523} &{0.894} & {0.489} & {0.910} &{0.532}\\
		\cline{1-14}
		\multirow{4}{*}{\makecell{\\ Ours\\(MLP)\\}}  & {$\sigma=0$} & {0.820} & {0.731} &{0.977} & {0.847} & {0.849} &{0.537}&{0.934} & {0.591} & {0.935} &{0.593}& {0.927} & {0.605}\\
		\cline{2-14}
		& {$\sigma=0.5$} & {0.784} & {0.814} &{0.979} & {0.844} & {0.845} &{0.533}&{0.936} & {0.584} & {0.933} &{0.601}& {0.929} & {0.588}\\
		\cline{2-14}
		& {$\sigma=1$}& {0.752} & {0.750} &{0.977} & {0.840} & {0.840} &{0.529}&{0.931} & {0.562} & {0.931} &{0.588}& {0.918} & {0.573}\\
		\cline{2-14}
		& {$\sigma=1.5$}& {0.730} & {0.763} &{0.978} & {0.846}  &{0.837} &{0.528}&{0.927} & {0.559} & {0.922} &{0.565} & {0.910} & {0.565}\\
		\hline
		
	\end{tabular}
	\caption{Comparison of seven deep learning baselines, single-parameter and our G-LoG bi-filtration on MedMNIST2D (\uppercase\expandafter{\romannumeral2}).}	
	\label{table_2d_2}
\end{table}

\begin{table}[H]\setlength{\tabcolsep}{0.6mm}
	\fontsize{6.5pt}{10.5pt}\selectfont
	\centering
	\begin{tabular}{|c|c|c|c|c|c|c|c|c|c|c|c|c|c|}
		\hline
		\multicolumn{2}{|c}{\multirow{1}{*}{Method}}
		&\multicolumn{2}{|c|}{OrganMNIST3D}&\multicolumn{2}{c|}{NoduleMNIST3D}
		&\multicolumn{2}{c|}{FractureMNIST3D}&\multicolumn{2}{c|}{AdrenalMNIST3D}&\multicolumn{2}{c|}{VesselMNIST3D}&\multicolumn{2}{c|}{SynapseMNIST3D}\\
		\cline{3-14} 	
		\multicolumn{2}{|c}{} & \multicolumn{1}{|c|}{AUC} & ACC &  AUC & ACC &  AUC & ACC &  AUC & ACC&  AUC & ACC&  AUC & ACC \\
		\cline{1-14}
		\multicolumn{2}{|c|}{ResNet-18+2.5D}  & {0.977} & {0.788} &{0.838} & {0.835} & {0.587} &{0.451} & {0.718} & {0.772} &{0.748} &{0.846} &{0.634}&{0.696}\\
		\cline{1-14}	
		\multicolumn{2}{|c|}{ResNet-18+3D}  & \textbf{0.996} & \textbf{0.907} &{0.863} & {0.844} & {0.712} &{0.508} & {0.827} & {0.721} &{0.874} &{0.877} &{0.820}&{0.745}\\
		\cline{1-14}
		\multicolumn{2}{|c|}{ResNet-18+ACS}  & {0.994} & {0.900} &{0.873} & {0.847} & {0.714} &{0.497}  & {0.839} & {0.754} &{0.930} &{0.928} &{0.705}&{0.722}\\
		\cline{1-14}	
		\multicolumn{2}{|c|}{ResNet-50+2.5D}  & {0.974} & {0.769} &{0.835} & {0.848} & {0.552} &{0.397} & {0.732} & {0.763} &{0.751} &{0.877} &{0.669}&{0.735}\\
		\cline{1-14}
		\multicolumn{2}{|c|}{ResNet-50+3D}  & {0.994} & {0.883} &{0.875} & {0.847} & {0.725} &{0.494} & {0.828} & {0.745} &{0.907} &{0.918} &\textbf{0.851}&{0.795}\\
		\cline{1-14}
		\multicolumn{2}{|c|}{ResNet-50+ACS}  & {0.994} & {0.889} &{0.886} & {0.841} & {0.750} &{0.517} & {0.828} & {0.758} &{0.912} &{0.858} &{0.719}&{0.709}\\
		\cline{1-14}
		\multicolumn{2}{|c|}{auto-sklearn}  & {0.977} & {0.814} &\textbf{0.914} & \textbf{0.874} & {0.628} &{0.453}& {0.828} & {0.802} &{0.910} &{0.915} &{0.631}&{0.730}\\
		\cline{1-14}
		\multicolumn{2}{|c|}{AutoKeras}  & {0.979} & {0.804} &{0.844} & {0.834} & {0.642} &{0.458}& {0.804} & {0.705} &{0.773} &{0.894} &{0.538}&{0.724} \\
		\cline{1-14}
		\multicolumn{2}{|c|}{Topo-Med (MLP)}  & {0.837} & {0.554} &{0.808} & {0.736} & {0.653} &{0.480} & {0.837} & {0.554} &{0.808} & {0.736} & {0.653} &{0.480}\\
		\cline{1-14}
		\multirow{4}{*}{\makecell{\\ Ours\\(MLP)\\}}  & {$\sigma=0$} & {0.958} & {0.684} &{0.564} & {0.774} & {0.697} &{0.554}  & {0.742} & {0.789} &{0.813} & {0.887} & {0.783} &{0.796}\\
		\cline{2-14}
		& {$\sigma=0.5$} & {0.961} & {0.702} &{0.789} & {0.852} & {0.704} &{0.579} & \textbf{0.870} & \textbf{0.847} &{0.877} &{0.898} &{0.810}&\textbf{0.827}\\
		\cline{2-14}
		& {$\sigma=1$} & {0.953} & {0.656} &{0.838} & {0.852} & {0.753} &{0.579} & {0.867} & {0.829} &{0.915} &{0.908} &{0.805}&{0.813}\\
		\cline{2-14}
		& {$\sigma=1.5$} & {0.941} & {0.634} &{0.803} & {0.848} & \textbf{0.774} &\textbf{0.588}  & {0.850} & {0.829} &\textbf{0.933} &\textbf{0.937} &{0.752}&{0.776}\\
		\hline
	\end{tabular}
	\caption{Comparison of eight deep learning baselines, single-parameter and our G-LoG bi-filtration on MedMNIST3D (\uppercase\expandafter{\romannumeral1}).}	
	\label{table_3d}
\end{table}

{\bf Results for 2D dataset} As shown in the Table \ref{table_2d_1} and Table \ref{table_2d_2}, our model outperforms single-parameter persistent homology across all datasets except for the AUC scores on ChestMNIST, DermaMNIST and RetinaMNIST. We observe a performance increase of 5–10\% on most datasets, with a notable 41.7\% increase in ACC on the Chest dataset.

Although our performance on 2D datasets does not surpass the best-performing baseline models, our results on the PathMNIST dataset achieves an AUC of 95.5\% and an ACC of 75.3\%, which outperform Auto-sklearn (93.4\% AUC and 71.6\% ACC). On the ChestMNIST dataset, our ACC reaches 94.7\%, surpassing Auto-sklearn (77.9\%) and AutoKeras (93.7\%). This performance is equivalent to that of  ResNet-18 (28), ResNet-18 (224) and ResNet-50 (28), trailing only behind ResNet-50 (224) and Google AutoML Vision (94.8\%). Regarding the BreastMNIST dataset, our ACC score of 81.4\% is higher than both ResNet-50 (28) (81.2\%) and Auto-sklearn (80.3\%). Finally, on the TissueMNIST dataset, our AUC and ACC scores (84.9\% and 53.7\%) outperform the results from Auto-sklearn (82.8\% AUC and 53.2\% ACC).

Across the majority of datasets, $\sigma=0.5$ yields superior results compared to $\sigma=0, 1 \text{ and } 1.5$. This empirically validates our motivation (2): the necessity of achieving an appropriate intersection of multi-parameter sublevel sets. These results underscore the latent potential of persistent homology, demonstrating that features extracted via this topological approach alone are sufficient to train models.

{\bf Results for 3D dataset} It can be observed in Table \ref{table_3d} that our method extracts superior features compared to single-parameter persistent homology, while our results remain highly competitive with the baseline methods. We select $\sigma=0, 0.5, 1, 1.5$ to generate bi-parameter persistence features. It is evident that the AUC and ACC scores for $\sigma=0$ are consistently lower than those for $\sigma=0.5, 1, 1.5$, which validates our initial motivation (2). Specifically, our method achieves an AUC of 77.4\% and an ACC of 58.8\% on FractureMNIST3D; 87.0\% and 84.7\% on AdrenalMNIST3D;  93.3\% and 93.7\% on VesselMNIST3D, respectively. On the SynapseMNIST3D dataset, our ACC reaches 82.7\%. Overall, our approach outperforms the baseline models in both AUC and ACC on the FractureMNIST3D, AdrenalMNIST3D and VesselMNIST3D datasets, while also surpassing the baseline in terms of ACC on the SynapseMNIST3D dataset.


Consequently, by selecting appropriate filter functions, multi-parameter persistent homology can extract sufficient geometric and topological features for classification tasks, potentially even surpassing the performance of original features.

\section{Conclusions and future work}

In this paper, we introduce a bi-parameter persistence filtration called G-LoG. We theoretically demonstrate that the resulting persistence modules on volumetric images are stable under the interleaving distance with respect to the maximum norm of the image data. The persistence modules induced by the G-LoG bi-filtration are computationally efficient to generate. Our experimental results highlight the significant potential of multi-parameter persistence modules in biomedical image analysis.

In the future, we aim to extend our methodology in two directions primarily. First, we plan to develop a more robust framework for  filtrations with more parameters, such as three-parameter filtrations, to capture even more intricate topological features. We believe that by selecting more appropriate filtration parameters, superior performance in multi-parameter persistence modules can be achieved. Second, we intend to integrate our bi-filtration framework into optimization pipelines, enabling its application to a broader range of domains, including computer graphics and end-to-end deep learning architectures.



\bibliographystyle{plain}
\bibliography{reference}

@article{SFHQLJCE-2023,
  title={Topological data analysis in medical imaging: current state of the art},
  author={Singh, Y. and Farrelly, C. M. and Hathaway, Q. A. and Leiner, T. and Jagtap, J. and Carlsson, G. E. and Erickson, B. J.},
  journal={Insights Imaging},
  volume={14},
  number={1},
  pages={58},
  year={2023},
  publisher={Springer}
}

@article{WHC-2021,
  title={Topological machine learning for mixed numeric and categorical data},
  author={Wu, C. and Hargreaves, C. A.},
  journal={Int. J. Artif. Intell. T.},
  volume={30},
  number={05},
  pages={2150025},
  year={2021},
  publisher={World Scientific}
}

@article{SHDSELB-2025,
  title={Quantifying the Unknowns of Plaque Morphology: The Role of Topological Uncertainty in Coronary Artery Disease},
  author={Singh, Y. and Hathaway, Q. A. and Dinakar, K. and Shaw, L. J. and Erickson, B. and Lopez-Jimenez, F. and Bhatt, D. L.},
  journal={Mayo Clinic Proceedings: Digital Health},
  volume={3},
  number={2},
  pages={100217},
  year={2025},
  publisher={Elsevier}
}

@article{SK-2022,
  title={Convolutional neural networks in medical image understanding: a survey},
  author={Sarvamangala, D. R. and Kulkarni, R. V.},
  journal={Evol. Intell.},
  volume={15},
  number={1},
  pages={1--22},
  year={2022},
  publisher={Springer}
}

@article{BDFFGLPS-2008,
author = {Biasotti, S. and De Floriani, L. and Falcidieno, B. and Frosini, P. and Giorgi, D. and Landi, C. and Papaleo, L. and Spagnuolo, M.},
title = {Describing shapes by geometrical-topological properties of real functions},
year = {2008},
issue_date = {October 2008},
publisher = {Association for Computing Machinery},
address = {New York, NY, USA},
volume = {40},
number = {4},
issn = {0360-0300},
url = {https://doi.org/10.1145/1391729.1391731},
doi = {10.1145/1391729.1391731},
journal = {ACM Comput. Surv.},
month = oct,
articleno = {12},
numpages = {87},
}

@inproceedings{HWFSC-2021,
title={Topology-Aware Segmentation Using Discrete Morse Theory},
author={Hu, X. and Wang, Y. and Fuxin, L. and Samaras, D. and Chen, C.},
booktitle={International Conference on Learning Representations},
year={2021},
}

@article{SFLLWW-2024knot,
  title={Knot data analysis using multiscale Gauss link integral},
  author={Shen, L. and Feng, H. and Li, F. and Lei, F. and Wu, J. and Wei, G.},
  journal={Proceedings of the National Academy of Sciences},
  volume={121},
  number={42},
  pages={e2408431121},
  year={2024},
  publisher={National Academy of Sciences}
}

@article {Edel-Lets-Zo-2000,
    AUTHOR = {Edelsbrunner, H. and Letscher, D. and Zomorodian, A.},
     TITLE = {Topological persistence and simplification},
   JOURNAL = {Discrete Comput. Geom.},
  FJOURNAL = {Discrete \& Computational Geometry. An International Journal
              of Mathematics and Computer Science},
    VOLUME = {28},
      YEAR = {2002},
    NUMBER = {4},
     PAGES = {\ 511--533},
      ISSN = {0179-5376,1432-0444},
   MRCLASS = {52B55 (65D18)},
  MRNUMBER = {1949898},
MRREVIEWER = {H.\ W.\ Guggenheimer},
       DOI = {10.1007/s00454-002-2885-2},
       URL = {https://doi.org/10.1007/s00454-002-2885-2},
}

@article{CL-2022,
  title={Persistence curves: A canonical framework for summarizing persistence diagrams},
  author={Chung, Y. M. and Lawson, A.},
  journal={Advances in Computational Mathematics},
  volume={48},
  number={1},
  pages={6},
  year={2022},
  publisher={Springer}
}

@article{Adams-Emer-Kir-2017,
  title={Persistence images: A stable vector representation of persistent homology},
  author={Adams, H. and Emerson, T. and Kirby, M. and Neville, R. and Peterson, C. and Shipman, P. and Chepushtanova, S. and Hanson, E. and Motta, F. and Ziegelmeier, L.},
  journal={J. Mach. Learn. Res.},
  volume={18},
  number={8},
  pages={\ 1--35},
  year={2017}
}

@inproceedings{RHBK-2015,
  title={A stable multi-scale kernel for topological machine learning},
  author={Reininghaus, J. and Huber, S. and Bauer, U. and Kwitt, R.},
  booktitle={Proceedings of the IEEE conference on computer vision and pattern recognition},
  pages={4741--4748},
  year={2015}
}

@article{DLZZL-2024,
  title={Persistence B-spline grids: stable vector representation of persistence diagrams based on data fitting},
  author={Dong, Z. and Lin, H. and Zhou, C. and Zhang, B. and Li, G.},
  journal={Mach. Learn.},
  volume={113},
  number={3},
  pages={1373--1420},
  year={2024},
  publisher={Springer}
}

@article{Peter-2015,
  AUTHOR = {Bubenik, P.},
     TITLE = {Statistical topological data analysis using persistence
              landscapes},
   JOURNAL = {J. Mach. Learn. Res.},
  FJOURNAL = {Journal of Machine Learning Research (JMLR)},
    VOLUME = {16},
      YEAR = {2015},
     PAGES = {77--102},
      ISSN = {1532-4435,1533-7928},
   MRCLASS = {62G99 (55N35 65D18 68U05)},
  MRNUMBER = {3317230},
MRREVIEWER = {Patrizio\ Frosini},
}

@article{CMCMR-2020,
  title={Predicting clinical outcomes in glioblastoma: an application of topological and functional data analysis},
  author={Crawford, L. and Monod, A. and Chen, A. X. and Mukherjee, S. and Rabad{\'a}n, R.},
  journal={J. Am. Stat. Assoc.},
  volume={115},
  number={531},
  pages={1139--1150},
  year={2020},
  publisher={Taylor \& Francis}
}

@article{CSKCTBABNM-2019,
  title={Network tomography for understanding phenotypic presentations in aortic stenosis},
  author={Casaclang-Verzosa, G. and Shrestha, S. and Khalil, M. J. and Cho, J. S. and Tokodi, M. and Balla, S. and Alkhouli, M. and Badhwar, V. and Narula, J. and Miller, J. D. and others},
  journal={JACC: Cardiovascular Imaging},
  volume={12},
  number={2},
  pages={236--248},
  year={2019},
  publisher={American College of Cardiology Foundation Washington, DC}
}

@inproceedings{YADGC-2023,
  title={Histopathological cancer detection with topological signatures},
  author={Yadav, A. and Ahmed, F. and Daescu, O. and Gedik, R. and Coskunuzer, B.},
  booktitle={2023 IEEE International Conference on Bioinformatics and Biomedicine (BIBM)},
  pages={1610--1619},
  year={2023},
  organization={IEEE}
}

@article{Carls-2009,
  AUTHOR = {Carlsson, G. and Zomorodian, A.},
     TITLE = {The theory of multidimensional persistence},
   JOURNAL = {Discrete Comput. Geom.},
  FJOURNAL = {Discrete \& Computational Geometry. An International Journal
              of Mathematics and Computer Science},
    VOLUME = {42},
      YEAR = {2009},
    NUMBER = {1},
     PAGES = {\ 71--93},
      ISSN = {0179-5376,1432-0444},
   MRCLASS = {52C35 (68U05)},
  MRNUMBER = {2506738},
       DOI = {10.1007/s00454-009-9176-0}
}

@article{OJPMUCHH-2021,
author = {O. Vipond  and J. A. Bull  and P. S. Macklin  and U. Tillmann  and C. W. Pugh  and H. M. Byrne  and H. A. Harrington },
title = {Multiparameter persistent homology landscapes identify immune cell spatial patterns in tumors},
journal = {P. Natl. Acad. Sci. USA},
volume = {118},
number = {41},
year = {2021},
doi = {10.1073/pnas.2102166118},
URL = {https://www.pnas.org/doi/abs/10.1073/pnas.2102166118},
}

@inproceedings{Math-2020,
  author = {Carrie\`re, M. and Blumberg, A.},
 booktitle = {Advances in Neural Information Processing Systems},
  pages = {22432--22444},
 publisher = {Curran Associates, Inc.},
 title = {Multiparameter Persistence Image for Topological Machine Learning},
 url = {https://proceedings.neurips.cc/paper_files/paper/2020/file/fdff71fcab656abfbefaabecab1a7f6d-Paper.pdf},
 volume = {33},
 year = {2020},
address ={New York}
}

@article{Re-2023,
  AUTHOR = {Corbet, R. and Kerber, M. and Lesnick, M. and
              Osang, G.},
     TITLE = {Computing the multicover bifiltration},
   JOURNAL = {Discrete Comput. Geom.},
  FJOURNAL = {Discrete \& Computational Geometry. An International Journal
              of Mathematics and Computer Science},
    VOLUME = {70},
      YEAR = {2023},
    NUMBER = {2},
     PAGES = {376--405},
      ISSN = {0179-5376,1432-0444},
   MRCLASS = {55N31},
  MRNUMBER = {4627346},
MRREVIEWER = {Massimo\ Ferri},
       DOI = {10.1007/s00454-022-00476-8}
}

@article{Blum-2022,
  AUTHOR = {Blumberg, A. J. and Lesnick, M.},
     TITLE = {Stability of 2-parameter persistent homology},
   JOURNAL = {Found. Comput. Math.},
  FJOURNAL = {Foundations of Computational Mathematics. The Journal of the
              Society for the Foundations of Computational Mathematics},
    VOLUME = {24},
      YEAR = {2024},
    NUMBER = {2},
     PAGES = {385--427},
      ISSN = {1615-3375,1615-3383},
   MRCLASS = {55N31 (62R40)},
  MRNUMBER = {4733354},
       DOI = {10.1007/s10208-022-09576-6}
}

@Article{JBTY-2025,
title = {Mix-GENEO: A flexible filtration for multiparameter persistent homology detects digital images},
journal = {AIMS Math.},
volume = {10},
number = {10},
pages = {24153-24178},
year = {2025},
issn = {2473-6988},
doi = {10.3934/math.20251071},
url = {https://www.aimspress.com/article/doi/10.3934/math.20251071},
author = {J. He and B. Hou and T. Wu and Y. Xin},
}

@article{Bergo-2019,
  title={Towards a topological--geometrical theory of group equivariant non-expansive operators for data analysis and machine learning},
  author={Bergomi, M. G. and Frosini, P. and Giorgi, D. and Quercioli, N.},
  journal={Nat. Mach. Intell.},
  volume={1},
  number={9},
  pages={\ 423--433},
  year={2019},
  publisher={Nature Publishing Group UK London}
}

@article{YSWLZKPN-2023,
  title={Medmnist v2-a large-scale lightweight benchmark for 2d and 3d biomedical image classification},
  author={Yang, J. and Shi, R. and Wei, D. and Liu, Z. and Zhao, L. and Ke, B. and Pfister, H. and Ni, B.},
  journal={Sci. Data},
  volume={10},
  number={1},
  pages={41},
  year={2023},
  publisher={Nature Publishing Group UK London}
}

@inproceedings{medmnistv1,
    title={MedMNIST Classification Decathlon: A Lightweight AutoML Benchmark for Medical Image Analysis},
    author={Yang, J. and Shi, R. and Ni, B.},
    booktitle={IEEE 18th International Symposium on Biomedical Imaging (ISBI)},
    pages={191--195},
    year={2021}
}

@inproceedings{HZRS-2016,
  title={Deep residual learning for image recognition},
  author={He, K. and Zhang, X. and Ren, S. and Sun, J.},
  booktitle={Proceedings of the IEEE conference on computer vision and pattern recognition},
  pages={770--778},
  year={2016}
}

@article{FEFLH-2022,
  title={Auto-sklearn 2.0: Hands-free automl via meta-learning},
  author={Feurer, M. and Eggensperger, K. and Falkner, S. and Lindauer, M. and Hutter, F.},
  journal={Journal of Machine Learning Research},
  volume={23},
  number={261},
  pages={1--61},
  year={2022}
}

@inproceedings{JSH-2019,
  title={Auto-keras: An efficient neural architecture search system},
  author={Jin, H. and Song, Q. and Hu, X.},
  booktitle={Proceedings of the 25th ACM SIGKDD international conference on knowledge discovery \& data mining},
  pages={1946--1956},
  year={2019}
}

@inproceedings{LLCLC-2022,
  title={Feature pyramid vision transformer for medmnist classification decathlon},
  author={Liu, J. and Li, Y. and Cao, G. and Liu, Y. and Cao, W.},
  booktitle={2022 International joint conference on neural networks (IJCNN)},
  pages={1--8},
  year={2022},
  organization={IEEE}
}

@article{MAKSA-2023,
  title={{MedViT: a robust vision transformer for generalized medical image classification}},
  author={Manzari, O. N. and Ahmadabadi, H. and Kashiani, H. and Shokouhi, S. B. and Ayatollahi, A.},
  journal={Comput. Biol. Med.},
  volume={157},
  pages={106791},
  year={2023},
  publisher={Elsevier}
}

@article{SJZ-2025,
  title={{Complex mixer for MedMNIST classification decathlon}},
  author={Sun, S. and Jia, X. and Zheng, Z.},
  journal={Appl. Intell.},
  volume={55},
  number={16},
  pages={1--12},
  year={2025},
  publisher={Springer}
}

@inproceedings{NKKC-2025,
  title={{Topological Machine Learning for Low Data Medical Imaging}},
  author={Nuwagira, B. and Caner, K. and Fan-Hsi, K. P.  and  Coskunuzer, B. },
  booktitle={Machine Learning for Health (ML4H)},
  pages={824--838},
  year={2025},
  organization={PMLR}
}

@article{Micheal-2015,
  AUTHOR = {Lesnick, M.},
     TITLE = {The theory of the interleaving distance on multidimensional
              persistence modules},
   JOURNAL = {Found. Comput. Math.},
  FJOURNAL = {Foundations of Computational Mathematics. The Journal of the
              Society for the Foundations of Computational Mathematics},
    VOLUME = {15},
      YEAR = {2015},
    NUMBER = {3},
     PAGES = {\ 613--650},
      ISSN = {1615-3375,1615-3383},
   MRCLASS = {55N35},
  MRNUMBER = {3348168},
MRREVIEWER = {Peter\ Bubenik},
       DOI = {10.1007/s10208-015-9255-y}
}

@book{CV-2021,
  title={Topological data analysis with applications},
  author={Carlsson, G. and Vejdemo-Johansson, M.},
  year={2021},
  publisher={Cambridge University Press},
  address = {Cambridge}
}

@book{Leo-2020,
  title={Topological persistence in geometry and analysis},
  AUTHOR = {Polterovich, L. and Rosen, D. and Samvelyan, K.
              and Zhang, J.},
     TITLE = {Topological persistence in geometry and analysis},
    SERIES = {University Lecture Series},
    VOLUME = {74},
 PUBLISHER = {American Mathematical Society},
      YEAR = {2020},
     PAGES = {xi+128},
      ISBN = {978-1-4704-5495-1},
   MRCLASS = {55N31 (53Dxx 58Cxx)},
  MRNUMBER = {4249570},
MRREVIEWER = {Walter\ D.\ Freyn},
ADDRESS = { RI}
}

@book{zomorodian2005topology,
  AUTHOR = {Zomorodian, A. J.},
  TITLE = {Topology for computing},
    SERIES = {Cambridge Monographs on Applied and Computational Mathematics},
    VOLUME = {16},
 PUBLISHER = {Cambridge University Press },
      YEAR = {2005},
     PAGES = {xiv+243},
      ISBN = {0-521-83666-2},
   MRCLASS = {68-01 (54H99 57M20 57Q05 57Q45 65D18 68U05)},
  MRNUMBER = {2111929},
MRREVIEWER = {Bert\ J\"uttler},
       DOI = {10.1017/CBO9780511546945},
ADDRESS = {Cambridge}
}

@article{LCB-2025,
  title={Multi-parameter Module Approximation: an efficient and interpretable invariant for multi-parameter persistence modules with guarantees},
  author={Loiseaux, D. and Carri{\`e}re, M. and Blumberg, A. J.},
  journal={J Appl. and Comput. Topology},
  volume={9},
  number={4},
  pages={1--60},
  year={2025},
  publisher={Springer}
}

@article{multipers-2024,
  title = {Multipers: {{Multiparameter Persistence}} for {{Machine Learning}}},
  shorttitle = {Multipers},
  author = {Loiseaux, D. and Schreiber, H.},
  year = {2024},
  month = nov,
  journal = {Journal of Open Source Software},
  volume = {9},
  number = {103},
  pages = {6773},
  issn = {2475-9066},
  doi = {10.21105/joss.06773},
  langid = {english},
}

@inproceedings{gudhi-2014,
  title={The gudhi library: Simplicial complexes and persistent homology},
  author={Maria, C. and Boissonnat, J. D. and Glisse, M. and Yvinec, M.},
  booktitle={Mathematical Software--ICMS 2014: 4th International Congress, Seoul, South Korea, August 5-9, 2014. Proceedings 4},
  pages={167--174},
  year={2014},
  organization={Springer}
}

@inproceedings{alonso2024delaunay,
  title={Delaunay bifiltrations of functions on point clouds},
  author={Alonso, {\'A}. J. and Kerber, M. and Lam, T. and Lesnick, M.},
  booktitle={Proceedings of the 2024 Annual ACM-SIAM Symposium on Discrete Algorithms (SODA)},
  pages={4872--4891},
  year={2024},
  organization={SIAM}
}

@article{lesnick2024nerve,
  title={Nerve Models of Subdivision Bifiltrations},
  author={Lesnick, M. and McCabe, K.},
  journal={arXiv preprint arXiv:2406.07679},
  year={2024}
}

\end{document}